\documentclass[11pt,table,reqno]{amsart}

\usepackage{amsfonts,amssymb,mathrsfs,kotex,graphicx,color,amscd,tikz,amsmath}
\usepackage[all]{xy}
\usepackage[onehalfspacing]{setspace}

\usepackage{todonotes}
\usepackage{hyperref}
\usepackage{ifthen}

\usetikzlibrary{calc}

\newtheorem{thm}{Theorem}[section]
\newtheorem{cor}[thm]{Corollary}
\newtheorem{lem}[thm]{Lemma}
\newtheorem{prop}[thm]{Proposition}

\theoremstyle{definition}
\newtheorem{definition}[thm]{Definition}
\newtheorem{exa}[thm]{Example}

\newcommand{\NP}{\textrm{NP}}
\newcommand{\Z}{\mathbb{Z}}
\newcommand{\R}{\mathbb{R}}

\usepackage{stmaryrd} 
\newcommand{\adj}{\leftrightarrow}
\newcommand{\adjeq}{\leftrightarroweq}

\let\int\relax
\DeclareMathOperator{\int}{int}
\let\Im\relax
\DeclareMathOperator{\Im}{Im}

\begin{document}

\baselineskip= 20.80pt

\title{Digital topological groups}

\thanks{
The first author was supported by the National Research Foundation of Korea (NRF) grant funded by the Korean government (MSIT) (No. 2018R1A2B6004407)}

\author{Dae-Woong Lee and P. Christopher Staecker}
\date{\today}

\address{
Department of Mathematics, and Institute of Pure and Applied Mathematics, Jeonbuk National University,
567 Baekje-daero, Deokjin-gu, Jeonju-si, Jeollabuk-do 54896, Republic of Korea
}
\email{dwlee@jbnu.ac.kr}

\address{
Mathematics Department, Fairfield University, 1703 North Benson Rd, Fairfield, CT 06824-5195
}
\email{cstaecker@fairfield.edu}


\begin{abstract} In this article, we develop the basic theory of digital topological groups. The basic definitions directly lead to two separate categories, based on the details of the continuity required of the group multiplication. We define $\NP_1$- and $\NP_2$-digital topological groups, and investigate their properties and algebraic structure. The $\NP_2$ category is very restrictive, and we give a complete classification of $\NP_2$-digital topological groups. We also give many examples of $\NP_1$-digital topological groups. 
We define digital topological group homomorphisms, and describe the digital counterpart of the first isomorphism theorem.
\end{abstract}

\maketitle

\section{Introduction}
The topological theory of digital images focuses on topological properties of finite subsets of the  $n$-dimensional integer lattice, attempting to build a digital theory which resembles classical topology of $\R^n$. 

A graph-theoretical approach to digital topology was introduced by A. Rosenfeld in the 1960s. Interesting works on digital surfaces and digital manifolds as a special kind of combinatorial manifolds have been produced in the 1980s and in 1990s by many authors.

A classical topological group is both a topological space and a group in which the topological structure is compatible with the group structure; that is, the product and inverse are topologically continuous. A topological group is both a special case of an H-space and a general case of a Lie group in pure mathematics.

%

A theory of H-spaces has been developed in the digital setting by O. Ege and I. Karaca \cite{EK} and the first author in \cite{dh1, dh2}. 
The additional algebraic structure of groups allows the present paper to prove more specific results. This paper is organized as follows: In Section \ref{pre}, we review and describe fundamental properties of digital images with some adjacency relations, digital continuous functions, and some terminology from graph theory. We define $\NP_i$-digital topological groups for some $i \in \{1,2\}$ and digital simple closed curves. In Section \ref{dtg}, we investigate properties of digital topological groups, and construct many classes of examples. In Section \ref{dtgh}, we develop digital topological group homomorphisms as morphisms of the category of digital topological groups. We also describe the digital counterpart of the first isomorphism theorem from group theory. In Section \ref{cdtg}, we describe a complete classification of $\NP_2$-digital topological groups. 
\bigskip

\section{Preliminaries} \label{pre}

A \emph{digital image} $(X, \kappa)$ consists of a finite set $X$ of points in $\Z^n$ with some \emph{adjacency relation} $\kappa$ which is both antireflexive and symmetric. 
Typically the adjacency relation is based on some notion of adjacency of points in $\Z^n$. This style of digital topology has its origins in the work of Rosenfeld and others, see \cite{R} for an early work.

When $n\geq 2$ there is no canonical ``standard adjacency'' to use in $\Z^n$ which corresponds naturally to the standard topology of $\R^n$. In the case of $\Z^2$, for example, at least two different adjacency relations are natural: we can view $\Z^2$ as a rectangular lattice connected by the coordinate grid, so that each point is adjacent to 4 neighbors; or we can additionally allow diagonal adjacencies so that each point is adjacent to 8 neighbors. 

In $\Z^n$, there are $n$ different standard adjacency relations, denoted $c_u$ for $u\in \{1,2,\dots,n\}$, defined as follows:
Two distinct points $x =(x_1,x_2,\ldots, x_n)$ and $y =(y_1,y_2,\ldots,y_n)$ in $\mathbb Z^n$ are {\it $c_u$-adjacent} \cite{LB3} if
\begin{itemize}
\item there are at most $u$ distinct indices $i$ such that $\vert x_i - y_i\vert =1$; and
\item for all other indices $j$, if $\vert x_j - y_j \vert \neq 1$, then $x_j = y_j$.
\end{itemize}

We will make use of the notation $x \adj_\kappa y$ when $x$ is adjacent to $y$ by the adjacency relation $\kappa$ in $\Z^n$, and $x \adjeq_\kappa y$ when $x$ is adjacent or equal to $y$. The particular adjacency relation will usually be clear from context, and in this case we will omit the subscript.

Let $a$ and $b$ be nonnegative integers with $a < b$. A {\it digital interval} \cite {LB0} is a finite set of the form
$$
[a,b]_{\mathbb{Z}} = \{ z \in \mathbb Z ~\vert~ a \leq z \leq b \},
$$
with the $c_1$-adjacency relation in $\mathbb{Z}$.

A digital image $(X,\kappa)$ is said to be {\it $\kappa$-connected} \cite {R} if for every pair $x,y$ of  points of $X$, there exists a subset
$\{ x_0 ,x_1, \ldots, x_s \}\subseteq X$
\begin{itemize}
\item $x=x_0$;
\item $x_s =y$; and 
\item $x_i \adjeq x_{i+1}$ for $i = 0,1,2,\ldots,s-1$ for $s \geq 1$.
\end{itemize}
Such a set $\{ x_0 ,x_1, \ldots, x_s \}$ is called a \emph{path} (or \emph{$\kappa$-path}) from $x$ to $y$. Given some $x\in X$, the set of all points of $X$ having a path to $x$ is the \emph{connected component} of $x$. A digital image is digitally connected if and only if it consists of a single component.

A function $f : (X,\kappa) \rightarrow (Y,\lambda)$ between digital images is called a {\it $(\kappa,\lambda)$-continuous function} \cite {LB1} when for any $x_1,x_2\in X$, if $x_1 \adj x_2$, then  $f(x_1) \adjeq f(x_2)$.
Equivalently, $f$ is continuous if
$f(C)$ is a $\lambda$-connected subset of $Y$ for each $\kappa$-connected subset $C$ of $X$. When the adjacencies are clear, we simply say that $f$ is \emph{continuous}.
Any composition of digitally continuous functions is digitally continuous.

A $(\kappa,\lambda)$-continuous function of digital images
$$
f : (X,\kappa) \rightarrow (Y,\lambda)
$$
is called a {\it $(\kappa,\lambda)$-isomorphism} if $f$ is a bijection set-theoretically, and its inverse $f^{-1} : (Y,\lambda) \rightarrow (X,\kappa)$ is $(\lambda,\kappa)$-continuous. In this case, $(X,\kappa)$ and $(Y,\lambda)$ are said to be {\it $(\kappa,\lambda)$-isomorphic}; see  \cite {LB0} and \cite{LB2}. We often make use of the following fact: if $f$ is an isomorphism, then $x\adj y$ if and only if $f(x)\adj f(y)$.

Any digital image $(X,\kappa)$ can naturally be viewed as a finite simple graph with vertex set $X$ and an edge connecting $x,y\in X$ if and only if $x\adj y$. Conversely, any graph may be considered as a digital image with a standard adjacency: 
\begin{thm}[\cite{los20}, Proposition 2.5]\label{embeddingthm}
Let $X$ be a finite simple graph of $k$ vertices. Then there is some $n < k$ such that $X$ may be embedded as a digital image $X \subset [-1,1]_\Z^n$ with $c_n$-adjacency.
\end{thm}
(Above, by ``embedded'' we mean that $X$, considered as an abstract graph, is isomorphic as a graph to the digital image $X\subset  [-1,1]_\Z^n$ with $c_n$-adjacency.)

We will, whenever convenient, describe a digital image $(X,\kappa)$ in graph-theoretic terms. Specifically, our definitions above of paths, connectedness, and components correspond exactly to the same concepts in graph theory. 
We will also make use of the \emph{degree} from graph theory: for $x\in X$, the degree of $x$ is the number of $y\in X$ with $y\adj x$. If all points of $X$ have the same degree $d$, we say $X$ is \emph{$d$-regular}.
If all pairs of points of $X$ are adjacent, then $X$ is a \emph{complete graph}.

Given two digital images $X_1$ and $X_2$, we can consider the Cartesian product $X_1\times X_2$ as a digital image, but there are several natural choices for the adjacency relations to be used in the Cartesian product, analogous to the various $c_u$-adjacencies. The most natural product adjacencies are the \emph{normal product adjacencies}, which were defined by Boxer as follows:

\begin{definition}(\cite{boxe17})
Let $\{(X_i,\kappa_i) \mid i = 1,2, \cdots, n\}$ be an indexed family of digital images.
Then for some $u \in \{1,2, \dots,n\}$, the \emph{normal product adjacency} $\NP_u(\kappa_1,\kappa_2,\dots,\kappa_n)$ is the adjacency relation on $\prod_{i=1}^n X_i$ defined by: $(x_1,x_2,\dots, x_n)$ and $(x'_1,x'_2, \dots, x'_n)$ are adjacent if and only if their coordinates are adjacent in at most $u$ positions, and equal in all other positions.
\end{definition}

In this paper, our products generally have the form $X\times X$ for some digital image $X := (X,\kappa)$. On such a Cartesian product, the two natural adjacencies to choose from are $\NP_1(\kappa,\kappa)$ and $\NP_2(\kappa,\kappa)$. When unambiguous, we will abbreviate $\NP_i(\kappa,\kappa)$ as simply $\NP_i$ for $i\in \{1,2\}$. Clearly, if two points $(x_1, x_2)$ and $(y_1, y_2)$ in $X\times X$ are $\NP_1$-adjacent, then they are also $\NP_2$-adjacent. As we will see, the topological structure of the Cartesian product depends strongly on the choice between $\NP_1$ and $\NP_2$; see \cite{los20,Chris2}
for a discussion of $\NP_2$-homotopy relations.

\bigskip

\section{Digital topological groups}\label{dtg}
\subsection{Definitions and basic examples}
A classical topological group is both a topological space and a group whose binary operation and inverse are topologically continuous. The following is the digital counterparts of classical topological groups as the topological structure of the Cartesian product depending strongly on the choice between $\NP_1$ and $\NP_2$.

\begin{definition}\label{dtgdefinition}
Let $(G,\kappa)$ be a digital image such that $G$ is a group with multiplication $\mu_G:G\times G \to G$ and inverse $\iota_G: G\to G$. For some $i\in \{1,2\}$, we say $G$ is an \emph{$\NP_i$-digital topological group} when $\iota_G$ is $(\kappa,\kappa)$-continuous and $\mu_G$ is $(\NP_i(\kappa,\kappa),\kappa)$-continuous.
If $G$ is an $\NP_i$-digital topological group for some $i\in \{1,2\}$, we say $G$ is a \emph{digital topological group}.
\end{definition}

For a digital topological group $G$ with product $\mu_G$ and inverse $\iota_G$ we typically will write the product of elements $x$ and $y$ as $\mu_G(x,y) = xy$, and the inverse of an element $x$ as $\iota_G(x) = x^{-1}$.

Clearly any $\NP_2$-digital topological group is also an $\NP_1$-digital topological group. In classical topology, the $\NP_2$ continuity condition is analogous to requiring that $\mu_G$ be continuous as a function of 2 variables, while the $\NP_1$ condition is analogous to requiring that $\mu_G$ is continuous in each variable separately. In classical topology the definition of topological group requires full continuity, while a structure with continuity in each variable separately is called a \emph{semitopological group} \cite[page 27]{TH}.

Classically, the assumption of full continuity is generally more interesting, but we will see in our setting that the more interesting examples of digital topological groups are $\NP_1$ but not $\NP_2$. The full $\NP_2$ continuity assumption is too strong to allow for any nontrivial examples, as we show in Section \ref{cdtg}.

The simplest nontrivial example of an $\NP_1$-digital topological group is a simple closed curve.

\begin{definition}\label{sccdef}
A \emph{digital simple closed curve} is 
a set $X : = \{x_0,x_1, \dots,x_{n-1}\}$ of distinct elements for $n \ge 3$ with some adjacency such that $x_i \adj x_j$ if and only if $j = i\pm 1$, with subscripts read modulo $n$.
\end{definition}

When viewed as a graph, all points of a digital simple closed curve have degree 2; that is, it is a connected $2$-regular graph. 

For example, the set $[-1,1]^2_\Z - \{(0,0)\}$ forms a digital simple closed curve in $(\Z^2,c_1)$. It is not a digital simple closed curve in $(\Z^2,c_2)$ because there are some extra diagonal adjacencies. Specifically the point $(1,0)$ is $c_2$-adjacent to its two $c_1$-neighbors $(1,1)$ and $(1,-1)$, but is also adjacent to $(0,1)$ and $(0,-1)$.

Now we show that a simple closed curve is an $\NP_1$-digital topological group, but not an $\NP_2$-digital topological group.

\begin{exa}\label{sccexa}
Let $G = \{x_0,x_1,\dots,x_{n-1}\}$ be a digital simple closed curve with adjacency $\kappa$, and define operations:
\[ \mu_G(x_i,x_j) = x_{i+j}, \qquad \iota_G(x_i) = x_{-i}, \]
where all subscripts are read modulo $n$.

We will show that these operations make $G$ into an $\NP_1$-digital topological group. The group axioms are clearly satisfied by these operations, so we need only check that $\mu_G$ is $(\NP_1,\kappa)$-continuous, and $\iota_G$ is $(\kappa, \kappa)$-continuous. 

For $\iota_G$, take two $\kappa$-adjacent points $x_i \adj x_{i+1}$. Then 
\[ \iota_G(x_i) = x_{-i} \adj x_{-i-1} = \iota_G(x_{i+1}) \]  
as desired. 

For $\mu_G$, we need to choose any two pairs of points in $G\times G$ which are $\NP_1$-adjacent, and show that they map by $\mu_G$ into $\kappa$-adjacent points of $G$. Points which are $\NP_1$-adjacent are equal in one coordinate, and adjacent in the other. So without loss of generality we will consider the pair $(x_i,x_j) \adj (x_{i+1},x_j)$. Then we have:
\[ \mu_G(x_i,x_j) = x_{i+j} \adj x_{i+j+1} = \mu_G(x_{i+1},x_j) \]
as desired. We have shown that $G$ is an $\NP_1$-digital topological group.

In contrast, these operations do not make $G$ into an $\NP_2$-digital topological group when $n>3$. To see this, note that $(x_0,x_0)$ and $(x_1,x_1)$ are $\NP_2$-adjacent, but their products $\mu_G(x_0,x_0) = x_0$ and $\mu_G(x_1,x_1) = x_2$ are not adjacent.
\end{exa}

\medskip

\subsection{Properties of digital topological groups}

In this subsection, we will investigate some interesting properties of digital topological groups. Most of the results in this section are inspired by corresponding statements in the classical theory. 

\begin{thm}\label{multiso}
Let $G = (G,\kappa)$ be a digital topological group.
Then, for each $x \in G$, the maps 
$\mu_x,\nu_x : G \to G$ 
given by
$\mu_x(y) = xy$ and $\nu_x(y) = yx$, 
respectively,
are $(\kappa,\kappa)$-isomorphisms. 
\end{thm}

\begin{proof}
Since the argument for $\nu_x$ is similar, we will prove the statement for $\mu_x$. Assume that $G$ is an $\NP_i$-digital topological group for some $i\in \{1,2\}$. 

First we show that $\mu_x$ is continuous for each $x \in G$. Let $f_x : G \to G\times G$ be defined by $f_x (y) = (x,y)$ for each $x \in G$. Note that $\mu_x = \mu_G \circ f_x$. Since $f_x$ is $(\kappa, \NP_j(\kappa,\kappa))$-continuous for both $j=1$ and $j=2$, and $\mu_G$ is $(\NP_i(\kappa,\kappa),\kappa)$-continuous, this means that $\mu_x$ will be $(\kappa,\kappa)$-continuous as desired.

By the same reason, $\mu_{x^{-1}}$ is continuous. But $\mu_{x^{-1}}$ is the inverse of $\mu_x$, and so $\mu_x$ is continuous with continuous inverse as desired.
\end{proof}


We can use the above to demonstrate that, in Definition \ref{dtgdefinition}, continuity of the inverse follows automatically from continuity of the product. (This is not generally true in the classical theory, in which it is possible to have a topological space with continuous group product but discontinuous inverse.)

\begin{lem}\label{inversecontinuous}
Let $(G,\kappa)$ be a digital image such that $G$ is a group with multiplication $\mu_G:G\times G \to G$ and inverse $\iota_G: G\to G$. For some $i\in \{1,2\}$, assume that $\mu_G$ is $\NP_i$-continuous for some $i\in \{1,2\}$. Then $G$ is an $\NP_i$-topological group.
\end{lem}
\begin{proof}
It suffices to show that the inverse is continuous: let $x\adj y$, and we will show that $x^{-1}\adj y^{-1}$.

Note that in Theorem \ref{multiso}, no reference was made to continuity of $\iota_G$. Thus we may use Theorem \ref{multiso} even though we have not assumed continuity of the inverse. 

Since $\mu_{x^{-1}}$ is an isomorphism by Theorem $\ref{multiso}$, we may apply $\mu_{x^{-1}}$ to $x\adj y$ to obtain $e \adj x^{-1}y$, where $e\in G$ is the identity element. Then we apply $\nu_{y^{-1}}$ and obtain $y^{-1}\adj x^{-1}$ as desired.
\end{proof}

We also have a partial converse to Theorem \ref{multiso}: if $\mu_x$ and $\nu_x$ are isomorphisms for every $x$, then $G$ is an $\NP_1$-digital topological group. 
\begin{thm}\label{munuthm}
Let $(G,\kappa)$ be a digital image such that $G$ is a group with multiplication $\mu_G:G\times G \to G$ and inverse $\iota_G: G\to G$. Assume that for each $x\in G$, the functions $\mu_x,\nu_x:G\to G$ are digital isomorphisms. Then $G$ is an $\NP_1$-digital topological group.
\end{thm}
\begin{proof}
By Lemma \ref{inversecontinuous}, we need only show that $\mu_G$ is $(\NP_1,\kappa)$-continuous. We must choose two pairs of $\NP_1$-adjacent points in $G$. Without loss of generality choose pairs $(x,y) \adj_{\NP_1} (x,z)$, and we will show that $xy\adj_\kappa xz$. Since $(x,y) \adj_{\NP_1} (x,z)$, we will have $y\adj z$, and applying the isomorphism $\mu_x$ gives $xy \adj_\kappa xz$ as desired.
\end{proof}

An easy corollary will weaken the assumption above:
\begin{cor}\label{onlycontinuous}
Let $(G,\kappa)$ be a digital image such that $G$ is a group with multiplication $\mu_G:G\times G \to G$ and inverse $\iota_G: G\to G$. Assume that for each $x\in G$, the functions $\mu_x,\nu_x:G\to G$ are digitally continuous. Then $G$ is an $\NP_1$-digital topological group.
\end{cor}
\begin{proof}
By Theorem \ref{munuthm}, we need only show that $\mu_x$ and $\nu_x$ are isomorphisms for every $x$ in $G$. Clearly $\mu_x$ and $\mu_{x^{-1}}$ are inverse functions. Since we have assumed that both are continuous, $\mu_x$ must be an isomorphism. Similarly $\nu_x$ is an isomorphism as desired.
\end{proof}

Since $\mu_x$ is an isomorphism which  carries the identity element $e$ to the element $x$, the structure of $G$ near $e$ must be isomorphic to the structure of $G$ near any other element $x$. 


A graph $G$ is \emph{vertex-transitive} if and only if, for all $x,y\in G$, there is a graph automorphism $f:G\to G$ with $f(x)=y$; see \cite[page 14]{West}, \cite[page 69]{CZ}, and \cite[page 15]{BM}.


\begin{thm}\label{regulargraph}
Let $G$ be a digital topological group. Then $G$ is vertex-transitive.
\end{thm}

\begin{proof}
Given a digital topological group $G$ and elements $x,y\in G$, by Theorem \ref{multiso} the map $$\mu_{yx^{-1}}:G\to G$$ is a graph automorphism carrying $x$ to $y$ as required. 
\end{proof}

Any vertex-transitive graph is automatically regular (all vertices have the same degree), so we have:
\begin{cor}
Let $G$ be a digital topological group. Then $G$ is regular.
\end{cor}

The results above are analogous to the classical statement that a topological group is a homogeneous space \cite[page 39]{MZ}.

We may also ask wether a digital topological group must be \emph{edge-transitive}, that is, there is a graph automorphism carrying any given edge onto any other given edge; see \cite[page 19]{BM} for a classical simple graph.
This is not necessarily true for a digital topological group, as we will show later in Example \ref{edgetransexa}.

The fact that a digital topological group must be regular is a significant restriction on the possible structure of a digital topological group. For example, among finite subsets of $(\Z^2,c_1)$, the only possible connected digital topological group is a digital simple closed curve.
\begin{thm}\label{Z2groups}
Let $G\subset (\Z^2,c_1)$ be a connected digital topological group. Then $G$ is a digital simple closed curve, or a set of 2 or fewer points.
\end{thm}
\begin{proof}
By Theorem \ref{regulargraph}, $G$ must be a $d$-regular graph for some $d$. Since we are using $c_1$-adjacency, the degree $d$ is at most 4. We will consider each possible value for $d\in \{0,1,2,3,4\}$, and conclude that the only allowable situation is when $G$ is a simple closed curve, or a set of 2 or fewer points.

If $d=0$ then $G$ is a single point. If $d=1$, then $G$ is a set of two adjacent points. For $d=2$, any connected $2$-regular graph is a simple cycle, which in our terminology is a simple closed curve. For $d=4$, the only $4$-regular subset of $(\Z^2,c_1)$ is all of $\Z^2$, and since $G$ is finite it cannot be $4$-regular.

It remains to show that the case $d=3$ is also impossible. To obtain a contradiction, assume that $G \subset (\Z^2,c_1)$ is a finite 3-regular graph. Since $G$ is finite we may compose with a translation to assume without loss of generality that $(0,0) \in G$, and that there are no points $(x,y)\in G$ with $x<0$. 

Since $(0,0)\in G$ but no point of $G$ has negative first coordinate, we have $(-1,0)\not \in G$ and thus since $(0,0)$ is degree 3, the neighbors of $(0,0)$ must be exactly $\{(0,1), (1,0), (0,-1)\}$. In particular, we have $(0,1)\in G$. Inductively applying the same reasoning to $(0,1)\in G$, we see that in fact $(0,y)\in G$ for every $y>0$, and so $G$ is infinite which is a contradiction.
\end{proof}

\medskip

\subsection{Examples of digital topological groups}
In this subsection, we present various example constructions of digital topological groups. 

Given any abstract group $G$, there are two trivial ways to realize $G$ as an $\NP_2$-digital topological group. First as a discrete digital image: we may embed $G\subset \Z^n$ as a set having no adjacent points. Then the required continuity conditions for the product and inverse are satisfied automatically, because there are no adjacencies in the domain. 

Second we can form an indiscrete digital image by embedding $G\subset \Z^n$ as a complete graph (by Theorem \ref{embeddingthm} this is possible for $n$ large enough). In this case the required continuity conditions are satisfied automatically because all points of the codomain are adjacent.

Generally we will be interested in examples other than these trivial cases of discrete and indiscrete digital images.

First we show that a direct product of digital topological groups is a digital topological group:
\begin{thm}
For some $i\in \{1,2\}$, let $(G,\kappa)$ and $(H,\lambda)$ be $\NP_i$-digital topological groups. Then $(G\times H, \NP_i(\kappa,\lambda))$ is an $\NP_i$-digital topological group.
\end{thm}
\begin{proof}
Let 
$
\mu_G:G\times G \to G
$
and 
$
\mu_H: H\times H \to H
$
be the multiplication operations for $G$ and $H$, respectively, and let $\iota_G : G \to G$ and $\iota_H : H \to H$ be the inverses for $G$ and $H$, respectively. Then we define the product and inverse 
$$
\mu_{G\times H}:(G\times H)\times(G\times H) \to G\times H
$$ 
and 
$$
\iota_{G\times H}: G\times H \to G\times H
$$
by
\[ 
\mu_{G\times H}((g,h),(g',h')) = (gg',hh'), 
\]
and
\[
\iota_{G\times H}(g,h) = (g^{-1},h^{-1}).
\]
These operations clearly obey the group axioms, so we need only demonstrate the appropriate continuities, which unfortunately are notationally cumbersome. We will write 
$$\gamma = \NP_i(\kappa,\lambda)
$$ 
for the adjacency relation being used in $G\times H$. 

By Lemma \ref{inversecontinuous}, we need only demonstrate that $\mu_{G\times H}$ is $(\NP_i(\gamma,\gamma), \gamma)$-continuous. Assume that $((g,h),(g',h'))$ and $((x,y),(x',y'))$, elements of $(G\times H)\times(G\times H)$, are $\NP_i(\gamma,\gamma)$-adjacent, and we must show that $(gg',hh')$ and $(xx',yy')$ are $\gamma$-adjacent. We will handle the $i=1$ and $i=2$ cases separately.

For $i=1$, we may assume without loss of generality that $(g,h)=(x,y)$ and $(g',h') \adj_\gamma (x',y')$. Further we can assume without loss of generality that $g'=x'$ and $h' \adj_{\lambda} y'$. Then when comparing $(gg',hh')$ with $(xx',yy')$, we see that the first coordinates are equal. In the second coordinate we have $h=y$ and $h' \adj_\lambda y'$, and thus $hh'\adj yy'$ since multiplication by $h=y$ is an isomorphism. Thus $(gg',hh')$ and $(xx',yy')$ are $\NP_1(\gamma,\gamma)$-adjacent as desired.

For $i=2$, we must additionally consider the case where $((g,h),(g',h'))$ and $((x,y),(x',y'))$ are $\gamma$-adjacent in each of the 2 coordinates. In this case we will have $g\adjeq_\kappa x$, $h\adjeq_\lambda y$, $g'\adjeq_\kappa x'$, and $h'\adjeq_\lambda y'$. By $\NP_2(\kappa,\kappa)$-continuity of $\mu_G$ and $\mu_H$ we have $gg' \adjeq xx'$ and $hh'\adjeq yy'$, and so $(gg',hh')$ and $(xx',hh')$ are $\NP_2(\gamma,\gamma)$-adjacent (or equal) as desired.
\end{proof}

The product result above can be used to construct examples of $\NP_1$-digital topological groups other than simple closed curves. For example if $X$ is a simple closed curve of length 4 and $Y$ is a simple closed curve of length 8, then $X\times Y$ is a digital topological group isomorphic to the direct product $\Z_4\times \Z_8$. This digital topological group is analogous to the classical torus $S^1\times S^1$ regarded as a topological group. 

This torus serves as an example digital topological group which is not edge-transitive. Recall that a graph is edge-transitive when, given any two edges, there is a graph automorphism carrying the first to the second. 
\begin{exa}\label{edgetransexa}
Let $G$ be the digital topological group which is a product of two digital simple closed curves of different lengths $k \neq l$. 

Viewing $G$ as a graph resembling the torus $S^1\times S^1$, it has some edges belonging to meridians of length $k$, and some belonging to meridians of length $l$. Because $k\neq l$, no automorphism can carry one of these edges to the other. 
\end{exa}

A large class of examples of $\NP_1$-digital topological groups is the set of Cayley graphs of groups. We will review the basic definitions as follows. 

Given a finitely generated group $G$, a \emph{group presentation} for $G$ is an expression $G = \langle S \mid R \rangle$, where $S$ is a generating set for $G$, and $R$ is a set of relations among the elements of $S$. For convenience we will assume that the generating set $S$ is ``symmetrized'', so that $g\in S$ implies $g^{-1}\in S$.

Given such a group presentation, we form the \emph{Cayley graph} \cite[page 28]{BM} in which the vertex set is $G$, and each element $x\in G$ is connected by an edge to every other element of the form $gx$ for each $g\in S$. The structure of the Cayley graph is determined by both the group $G$ and the set $S$: the same group may give different nonisomorphic Cayley graphs depending on the choice of generating set.

\begin{thm}\label{cayleythm}
Let $X$ be a Cayley graph for some group presentation $G = \langle S \mid R \rangle$. Then $X$ is an $\NP_1$-digital topological group which is isomorphic as a group to $G$.
\end{thm}
\begin{proof}
The vertices of $X$ are elements of $G$, so there is a natural multiplication $\mu_X:X\times X\to X$ and inverse $\iota_X:X\to X$ which satisfy the group axioms. Using these operations, $X$ is isomorphic to $G$ as a group. To show that $X$ is an $\NP_1$-digital topological group, by Lemma \ref{inversecontinuous} we must show that $\mu_X$ is continuous using $\NP_1$ in the domain.

Let $(x,y)$ and $(u,v)$ be $\NP_1$-adjacent, and we will show that $xy\adj uv$ in $X$. Without loss of generality, we may assume that $x=u$ and $y\adj v$. Since $X$ is a Cayley graph, $y\adj v$ means that $v = gy$ for some $g \in S$. So we must show that $xy \adj xgy$.

By the construction of the Cayley graph we have $e\adj g$ since $g\in S$. Then applying the isomorphisms $\mu_x$ and $\nu_y$ gives $xy \adj xgy$ as desired.
\end{proof}

As a specific example, it is easy to construct the dihedral group $D_{8}$ on 8 elements as an $\NP_1$-digital topological group in $(\Z^3,c_1)$ modeled on the unit cube. This is notably an example of a non-abelian digital topological group with a nontrivial (neither discrete nor indiscrete) graph structure. 
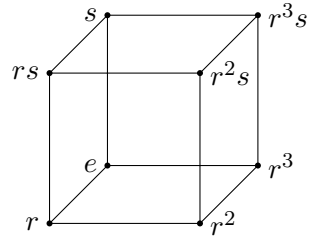
\begin{figure}[h]
\begin{tikzpicture}[scale=2]
\draw (0,0,0) -- (1,0,0) -- (1,1,0) -- (0,1,0) -- (0,0,0) -- (0,0,1) -- (1,0,1) -- (1,1,1) -- (0,1,1) -- (0,0,1);
\draw (1,0,0) -- (1,0,1);
\draw (1,1,0) -- (1,1,1);
\draw (0,1,0) -- (0,1,1);
\node[left] at (0,0,0) {$e$};
\node[left] at (0,0,1) {$r$};
\node[right] at (1,0,1) {$r^2$};
\node[right] at (1,0,0) {$r^3$};
\node[left] at (0,1,0) {$s$};
\node[left] at (0,1,1) {$rs$};
\node[right] at (1,1,1) {$r^2s$};
\node[right] at (1,1,0) {$r^3s$};
\foreach \x in {0,1} {
 \foreach \y in {0,1} {
  \foreach \z in {0,1} {
   \fill (\x,\y,\z) circle (0.02);
   }}}
\end{tikzpicture}
\caption{\label{d8fig} A non-abelian digital topological group modeled on the dihedral group $D_8$.}
\end{figure}

\begin{exa}\label{d8example}
Let $G \subset (\Z^3,c_1)$ be the unit cube $G = [0,1]_\Z^3$. Our definition of the operations is inspired by the Cayley graph of the presentation $D_8 = \langle r, s \mid r^4=e, s^2=e, srs^{-1}=r^{-1} \rangle$. We identify points of $G$ with elements of $D_8$ according to Figure \ref{d8fig}, and define operations 
$$
\mu_G : G \times G \rightarrow G
$$ 
and 
$$
\iota_G : G \rightarrow G
$$ 
by the structure of $D_8$; see Tables \ref{LStable10} and \ref{LStable11}.

\renewcommand{\tabcolsep}{25.48pt}
\renewcommand{\arraystretch}{1.50}
\begin{table}[h!]
  \begin{center}
\begin{tabular}{c| c c c c c}
   \hline
      {$\mu_G$}  & $e$   & $r$   & $r^2$ & $r^3$  & $s$ \\	    \hline
      $e$        & $e$   & $r$   & $r^2$ & $r^3$  & $s$ \\
      $r$        & $r$   & $r^2$ & $r^3$ & $e$    & $rs$   \\
      $r^2$      & $r^2$ & $r^3$ & $e$   & $r$    & $r^2s$   \\
      $r^3$      & $r^3$ & $e$   & $r$   & $r^2$  & $r^3s$   \\
      $s$        & $s$   & $r^3s$  & $r^2s$ & $rs$  & $e$ \\
   \hline
\end{tabular}
\vspace*{4mm}
\end{center}
 \caption{The multiplication $\mu_G : G \times G \rightarrow G$.} \label{LStable10}
\end{table}

\renewcommand{\tabcolsep}{16.48pt}
\renewcommand{\arraystretch}{1.50}
\begin{table}[h!]
  \begin{center}
\begin{tabular}{c| c c c c c c c c}
   \hline
      {}          & $e$   & $r$     & $r^2$   & $r^3$ & $s$ & $rs$   & $r^2s$ & $r^3s$ \\	    \hline
      {$\iota_G$} & $e$   & $r^3$   & $r^2$   & $r$ & $s$ & $sr^3$ & $sr^2$ & $sr$ \\
     
   \hline
\end{tabular}
\vspace*{4mm}
\end{center}
 \caption{The inverse operation $\iota_G : G \rightarrow G$.} \label{LStable11}
\end{table}
\end{exa}

A typical Cayley graph will not be an $\NP_2$-digital topological group. The assumption below will typically be satisfied by any Cayley graph, as long as there is some element of order greater than 2.
\begin{thm}
Let $X$ be a Cayley graph for some group presentation $G = \langle S \mid R \rangle$ having some generator $x\in S$ such that $x^2\not\in S$ and $x^2 \neq e$. Then $X$ is not an $\NP_2$-digital topological group.
\end{thm}
\begin{proof}
Since $x\in S$ we will have $e\adj x$, and so $(e,e)$ and $(x,x)$ are $\NP_2$-adjacent. But $\mu_X(e,e) = e$ and $\mu_X(x,x)=x^2$ are not equal or adjacent because $x^2\not\in S$ and $x^2 \neq e$ was assumed. Thus $\mu_X$ is not an $\NP_2$-continuous function.
\end{proof}

It is natural to consider the converse question to Theorem \ref{cayleythm}: if $G$ is an $\NP_1$-digital topological group, then is $G$ a Cayley graph for some group isomorphic to $G$?
The answer is no, as the following example demonstrates:

\begin{exa}
Let $G \subset \Z^2$ be the unit square of four points:
\[ x_0 = (0,0), x_1=(1,0), x_2 = (1,1),x_3 = (0,1), \]
and define operations 
\[
\mu_G(x_i,x_j) = x_{i+j}
\] 
and 
\[
\iota_G(x_i) = x_{-i}
\] 
with subscripts read modulo $4$, so that $G$ is isomorphic as a group to the cyclic group $\Z_4$ of order 4.

The operation above makes $(G,c_2)$ into a digital topological group, where the continuity of $\mu_G$ and $\iota_G$ is automatic because $(G,c_2)$ is a complete graph. But $G$ is not the Cayley graph for $\Z_4$, which would be a simple cycle graph of 4 points. 
\end{exa}


\subsection{Connected components of a digital topological group}

A basic result from the classical theory of topological groups is that the component $G_e$ of the identity element $e$ of a topological group $G$ is a normal subgroup; see \cite[page 39]{MZ}.
The same is true in the digital setting; see Theorem \ref{normal} below.


\begin{definition}
Let $G$ be a digital topological group with identity element $e\in G$. Define $G_e$ to be the connected component containing the identity element $e$, which we call the \emph{identity component}.

More generally, for some $x\in G$, let $G_x$ be the component containing $x$.
\end{definition}

\begin{lem}\label{mucomponents}
Let $G$ be a digital topological group. For $x,y\in G$, we have $\mu_x(G_y) = G_{xy}$ and $\nu_x(G_y) = G_{yx}$.
\end{lem}
\begin{proof}
Since the statement for $\nu_x$ is similar, we will prove the statement for $\mu_x$.  
Note that 
$$
\mu_x(y) = xy \in G_{xy}.
$$
Thus $\mu_x$ maps $y$ into $G_{xy}$, and since the continuous image of a digital connected subset is connected, we must have $$\mu_x(G_y) \subseteq G_{xy}.$$ To see that these sets are equal, note that $\mu_{x^{-1}}$ maps $xy$ into $G_y$, and so by the same reasons as above we have $$\mu_{x^{-1}}(G_{xy}) \subseteq G_y.$$ Applying $\mu_x$ gives $G_{xy} \subseteq \mu_x(G_y)$, which combines with the above to give $$\mu_x(G_y) = G_{xy}$$
as required.\end{proof}

Since $\mu_x$ is a digital isomorphism, the above gives:
\begin{cor}\label{componentsisom}
Let $G$ be a digital topological group. For $x,y\in G$, the sets $G_x$ and $G_y$ are digitally isomorphic.
\end{cor}

The above means that any digital topological group $G$ naturally decomposes as a disjoint union of connected components, each isomorphic to $G_e$. Specifically, any component $G_x$ is digitally isomorphic to $G_e$ by the isomorphism $\mu_x:G_e \to G_x$. 

\begin{thm}\label{normal}
Let $G$ be a digital topological group. Then $G_e$ is a normal subgroup of $G$.
\end{thm}
\begin{proof}
To show that $G_e$ is a subgroup, take $x,y\in G_e$, and we will show that $x^{-1}y \in G_e$. Since $x$ and $y$ are in the same component, we can construct a path from $x$ to $y$ in $G_e$. Applying the isomorphism $\mu_{x^{-1}}$ to this path gives a path from $e$ to $x^{-1}y$, which means that $x^{-1}y\in G_e$ as desired. 

To show that $G_e$ is normal, take $x\in G$ and $y\in G_e$, and we must show that $xyx^{-1}\in G_e$. Since $y\in G_e$, by Lemma \ref{mucomponents} we have $\mu_x(y) \in G_x$, and thus $xyx^{-1} = \nu_{x^{-1}}(\mu_x(y)) \in G_e$ as desired.
\end{proof}

Because $G_e$ is a connected subgroup and the whole group $G$ consists of disjoint copies of $G_e$, our task of classifying digital topological groups will focus on connected groups.

Let $C_G$ be the set of all (connected) components of a digital topological group $G$. 
Suggested by Lemma \ref{mucomponents}, there is a natural group structure on the set $C_G$ of components of $G$ as follows. 

\begin{thm}
Let $G$ be a digital topological group. Then $C_G$ is a group, with operations given by $G_x\cdot G_y = G_{xy}$ and $(G_x)^{-1} = G_{x^{-1}}$, where $x,y \in G$.
\end{thm}
\begin{proof}
By the group properties of $G$, these operations on $C_G$ will satisfy the group axioms. It suffices only to show that these operations are well-defined.

To show that the inverse is well defined, take $x,x'$ in $G$ with $G_x = G_{x'}$, and we will show that $G_{x^{-1}} = G_{x'^{-1}}$. Since $G_x = G_{x'}$, there is a path from $x$ to $x'$ in $G$. Since the inverse map
$$
\iota_G : G \rightarrow G
$$ 
is continuous, we may apply $\iota_G$ to this path to obtain a path from $x^{-1}$ to $x'^{-1}$, which shows that $x^{-1}$ and $x'^{-1}$ are in the same component, and thus that $G_{x^{-1}} = G_{x'^{-1}}$ as desired.

Now we show that the product $\cdot$ is well-defined. To do this, take $x,y,x',y' \in G$ with $G_x = G_{x'}$ and $G_y = G_{y'}$, and we must show that $G_{xy} = G_{x'y'}$. That is, we must show that: when $x$ and $x'$ are in the same component, and when $y$ and $y'$ are in the same component, then $xy$ and $x'y'$ are in the same component.

Since $x$ and $x'$ are in the same component, there is a path $$x=x_0,x_1,\dots,x_n=x'$$ from $x$ to $x'$ in $G$. Similarly there is a path $$y=y_0,y_1,\dots,y_m=y'$$ from $y$ to $y'$ in $G$. By perhaps repeating elements several times to lengthen the paths, we may assume that $m=n$ so that these paths have the same length. We will also assume that points are repeated in these paths in the following way: 
\begin{itemize}
\item for even $i$ we assume that $x_i=x_{i+1}$, and 
\item or odd $i$ we assume that $y_i = y_{i+1}$. 
\end{itemize}

Now we may multiply these paths pointwise to obtain:
\[ xy = x_0y_0, x_1y_1, \dots, x_ny_n=x'y' \]
and it suffices to show that this sequence of points forms a path from $xy$ to $x'y'$. We must demonstrate that $$x_iy_i \adjeq x_{i+1}y_{i+1}$$ for each $i<n$. Because of the structure of repeated values in even and odd positions, we see that $(x_i,y_i)$ and $(x_{i+1},y_{i+1})$ are $\NP_1$-adjacent for any $i<n$. Thus their products $x_iy_i$ and $x_{i+1}y_{i+1}$ are adjacent as desired.
\end{proof}

\begin{thm}\label{LSnonsplit}
Let $G$ be a digital topological group with identity element $e$. Then there is a group isomorphism $G / G_e \cong C_G$.
\end{thm}
\begin{proof}
For $x\in G$, let $\bar x\in G/G_e$ be the representative of $x$ in the quotient. Let 
$$
f:G/G_e \to C_G
$$ 
be given by 
$$
f(\bar x) = G_{x}.
$$
We will show that $f$ is an isomorphism.

First we show that $f$ is well-defined. Let $\bar x = \bar y$, and we must show that $G_{x} = G_{y}$, that is, $x$ and $y$ are in the same component. The fact that $\bar x=\bar y$ means that $ x = g y$ for some $g\in G_e$. Since $g\in G_e$ and $y\in G_y$, we have $gy \in G_y$ by Lemma \ref{mucomponents}. Thus $x \in G_y$, and so $x$ and $y$ are in the same component as desired.

Clearly $f$ is a group homomorphism because $$f(\bar x\bar y) = G_{xy} = G_x \cdot G_y = f(\bar x)f(\bar y)$$ and $$f(\bar x^{-1}) = G_{x^{-1}} = (G_x)^{-1} = f(\bar x)^{-1}.$$ 

Finally we show that $f$ is a bijection. The kernel of $f$ is any $\bar x$ with $G_x = G_e$, that is to say $\bar x \in \ker f$ if and only if $x \in G_e$, which holds only when $x$ is trivial in $G/G_e$. Thus $f$ has trivial kernel. Finally $f$ is surjective because for any element $G_x\in C_G$, we have 
$$
f(\bar x) = G_x
$$
as required.
\end{proof}

Note that Theorem \ref{LSnonsplit} asserts that there is a short exact sequence of groups
\begin{eqnarray} \label{LSsplit}
\xymatrix@C=15mm @R=10mm{
0  \rightarrow G_e  \hookrightarrow G   \twoheadrightarrow   C_G  \rightarrow 0
}
\end{eqnarray}
that is not necessarily split. If the sequence (\ref{LSsplit}) happens to be split, then $G$ is a semidirect product of $C_G$ and $G_e$; that is,
$$
G \cong G_e \rtimes C_G.
$$

In particular, it is not always true that $G$ is algebraically the direct product $G \cong G_e \times C_G$, as the following example shows:

\begin{exa}
Let $G$ be the dihedral group $D_8$ (see Example \ref{d8example}), this time modeled on the set 
\[ G = (\{0,1\}\times \{0,1\}) \cup (\{3,4\}\times \{0,1\}) \subset (\Z^2,c_1) \]
as shown in Figure \ref{disconnectedd8fig}. 
\begin{figure}[h]
\[ 
\begin{tikzpicture}[scale=2]
\draw (0,0) grid (1,1);
\draw (3,0) grid (4,1);

\node[left] at (0,0) {$e$};
\node[right] at (1,0) {$r$};
\node[right] at (1,1) {$r^2$};
\node[left] at (0,1) {$r^3$};
\node[left] at (3,0) {$s$};
\node[right] at (4,0) {$rs$};
\node[right] at (4,1) {$r^2s$};
\node[left] at (3,1) {$r^3s$};

\foreach \x in {0,1,3,4} {
 \foreach \y in {0,1} {
   \fill (\x,\y) circle (0.02);
   }}
\end{tikzpicture}
\]
\caption{Another realization of $D_8$ as an $\NP_1$-digital topological group.\label{disconnectedd8fig}}
\end{figure}
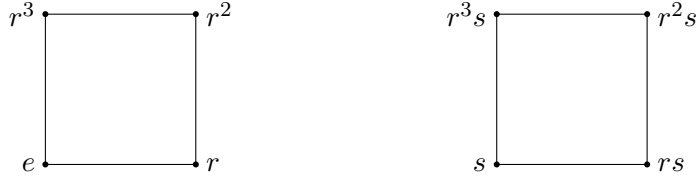
We must show that the product operation used in Example \ref{d8example} is continuous with respect to the adjacencies in the figure. 

An arbitrary point of $G$ has the form $r^ks^\epsilon$, where $k\in \{0,1,2,3\}$ and $\epsilon \in \{0,1\}$. To show that the product is continuous, we choose two pairs of $\NP_1$-adjacent points. Without loss of generality, we take the pairs $(r^ks^\epsilon, r^ls^\delta)$ and $(r^ks^\epsilon, r^{l+1}s^\delta)$ where $k,l\in \{0,1,2,3\}$ and $\epsilon,\delta \in \{0,1\}$, where exponents of $r$ are always read modulo 4, and exponents of $s$ are always read modulo 2. We must show that the products $r^ks^\epsilon r^ls^\delta$ and $r^ks^\epsilon r^{l+1}s^\delta$ are adjacent in $G$. 

When $\epsilon=0$, we have
\[ r^ks^\epsilon r^ls^\delta = r^{k+l}s^{\delta} \adj r^{k+l+1}s^{\delta} = r^k s^\epsilon r^{l+1} s^\delta \]
as desired. Recall that in $D_8$ we have $srs^{-1} = r^{-1}$, which means that $sr = r^{-1}s$ and more generally that $sr^k = r^{-k}s$.
Thus when $\epsilon = 1$, we have
\[ r^ks^\epsilon r^l s^\delta = r^k r^{-l} s s^\delta \adj r^k r^{-l-1} s s^{\delta} = r^k s^\epsilon r^{l+1}s^\delta \]
as desired.

Thus $G$ is an $\NP_1$-digital topological group. Observe that $G_e$ is isomorphic as a group to $\Z_4$, while $C_G$ is isomorphic to the group $\Z_2$, and so $G$ is not the direct product of $G_e$ and $C_G$. Rather, it is the semidirect product $D_8 \cong \Z_4 \rtimes \Z_2$, where the $\Z_2$ factor acts by inversion.

This example also demonstrates that digital topological groups which are isomorphic as groups may not be isomorphic as digital images. The present example and Example \ref{d8example} are algebraically the same group, but realized as two different digital images.
\end{exa}

\bigskip

\section{Digital topological group homomorphisms} \label{dtgh}

In this section, we consider morphisms between digital topological groups, and the first isomorphism theorem for digital topological groups. 
We define a morphism between digital topological groups as follows.

\begin{definition}
Let $G = (G, \kappa)$ and $H = (H, \lambda)$ be digital topological groups. A map $f : G \rightarrow H$ is called a {\it digital topological group homomorphism} if it is both a $(\kappa,\lambda)$-continuous function and a group homomorphism.
\end{definition}


\begin{definition}
A digital topological group homomorphism $f : G \rightarrow H$ is called a {\it digital topological group isomorphism} if it is a digital isomorphism; that is, there exists a digital topological group homomorphism $f^{-1} : H \rightarrow G$ such that
$$
f^{-1} \circ f = 1_G
$$
and
$$
f \circ f^{-1} = 1_H.
$$
In this case, $G$ is said to be {\it digital topological group isomorphic} to $H$.
\end{definition}

Let $N$ be a subgroup, group-theoretically, of a digital topological group $G$.  
We define an adjacency relation on $N$ as follows.

\begin{definition}\label{subsetdef}
Let $N$ be a subgroup of a digital topological group $G = (G, \kappa)$. We define an adjacency relation $\kappa\vert_{N}$ on $N$ by the restriction of $\kappa$ to $N$. Using this restricted adajcency relation, $N$ becomes a digital topological group, which we call a \emph{digital topological subgroup} of $G$.
\end{definition}

%
%

Note that a digital topological subgroup $(N,\kappa\vert_N)$ of $(G, \kappa)$ is not necessarily digital connected even if $(G, \kappa)$ is $\kappa$-connected.

Let $N$ be a normal digital topological subgroup of a digital topological group $G$.  
We define an adjacency relation on $G/N$ as follows.

\begin{definition}\label{quodef}
Let $G = (G, \kappa)$ be a digital topological group and $N$ a normal digital topological subgroup, and let $\pi:G \to G/N$ be the canonical surjection. We define an adjacency relation $\bar \kappa$ on $G/N$ by declaring $\bar x \adj_{\bar \kappa} \bar y$ whenever there are $x\in \pi^{-1}(\bar x)$ and $y \in \pi^{-1}(\bar y)$ with $x \adj_\kappa y$ in $G$.  
\end{definition}

\begin{prop}\label{LS101}
Let $(G,\kappa)$ be an $\NP_i$-digital topological group, with an $\NP_i$-digital topological subgroup $N$ for some $i \in \{1,2\}$. Then the quotient group $G/N$ is an $\NP_i$-digital topological group with adjacency relation $\bar \kappa$ as in Definition \ref{quodef}.
\end{prop}

\begin{proof}
We note that the projection map
$$
\pi : (G, \kappa) \rightarrow (G/K, \bar \kappa)
$$
given by
$
\pi (x) = \bar x
$
is $(\kappa, \bar\kappa)$-continuous by Definition \ref{quodef}; that is, if 
$
x \adj_\kappa y
$ 
in $G$, then 
$
\bar x \adj_{\bar \kappa} \bar y.
$

Let $\mu_G : G \times G \rightarrow G$ and $\iota_G : G \rightarrow G$ be the digital multiplication and the digital inverse, respectively, as in Definition \ref{dtgdefinition}. 
Define 
$$
\mu_{G/N} : G/N \times G/N \rightarrow G/N
$$
and
$$
\iota_{G/N} : G/N \rightarrow G/N
$$
by
$$
\mu_{G/N} (\bar x, \bar y) = \overline{xy} 
$$ 
and
$$
\iota_{G/N} (\bar x) = \overline{x^{-1}},
$$ 
respectively, for all $x,y \in G$. 
It can be seen that the following diagrams 
$$
\xymatrix@C=13mm @R=11mm{
G \times G \ar[d]_-{\pi \times \pi} \ar[r]^-{\mu_G} &G  \ar[d]^-{\pi} &\mathrm{and} & G \ar[d]_-{\pi} \ar[r]^-{\iota_G} &G  \ar[d]^-{\pi}\\ 
G/N \times G/N  \ar[r]^-{\mu_{G/N}} &G/N  && G/N  \ar[r]^-{\iota_{G/N}} &G/N
}
$$
are commutative. Moreover, if 
$$
(\bar x_1, \bar y_1) \adj_{\NP_i (\bar \kappa, \bar \kappa)} (\bar x_2, \bar y_2)
$$ 
in $G/N \times G/N$ for some $i \in \{1,2\}$, and
$$
\bar z_1 \adj_{\bar \kappa} \bar z_2
$$ 
in $G/N$, 
then we have
$$
\mu_{G/N} (\bar x_1, \bar y_1) = \overline{x_1 y_1} \adj_{\bar \kappa} \overline{x_2 y_2} = \mu_{G/N} (\bar x_2, \bar y_2)
$$
and
$$
\iota_{G/N}(\bar z_1 ) = \overline{z_1} \adj_{\bar \kappa} \overline{z_2} = \iota_{G/N}(\bar z_2 ). 
$$
Therefore, $\mu_{G/N}$ and $\iota_{G/N}$ are $(\NP_i (\bar \kappa, \bar\kappa), \bar \kappa)$-continuous and $(\bar \kappa, \bar \kappa)$-continuous functions, respectively, for some $i \in \{1,2\}$, as required.
\end{proof}

We note that, when $G$ is a digital topological group embedded in $(\Z^n,c_u)$ and $N$ is a normal subgroup, the quotient group $G/N$ may not nicely take the form of a digital image in $\Z^n$ with $c_u$ adjacency. 

For example, let $G \subset (\Z^2,c_1)$ be the simple closed curve of 16 points in Figure \ref{quotientfig}. We will label the points of $G$ as $x_0,x_1,\dots,x_{15}$ with $x_i \adj x_{i+1}$ with subscripts read modulo $16$, so that $G$ is isomorphic as a group to $\Z_{16}$. Let $N = \{x_0,x_4,x_8,x_{12}\}$, indicated by the circled points in Figure \ref{quotientfig}. Choosing $\{\bar x_0,\bar x_1,\bar x_2,\bar x_3\}$ as coset representatives, the digital image $(G/N,\bar c_1)$ is not naturally situated as a digital image in $\Z^2$ with $c_1$-adjacency. (Though by Theorem \ref{embeddingthm}, it will be isomorphic to some digital image in $\Z^n$ with $c_n$-adjacency. In fact it is a simple closed curve of 4 points.)

\begin{figure}[h]
\begin{tikzpicture}[scale=.7]
\draw[densely dotted] (-3,-3) grid (3,3);
\draw (-2,-2) rectangle (2,2);
\foreach \i in {-2,2} {
 \foreach \j in {-2,...,2} {
   \fill (\i,\j) circle (0.08);
   \fill (\j,\i) circle (0.08);
   }
}
\foreach \p in {(2,0), (0,2), (-2,0), (0,-2)} {
 \draw \p circle (0.16);
}
\node[below right] at (2,0) {$x_0$};
\node () at (0,-4) {(a)};
\end{tikzpicture}
\qquad
\begin{tikzpicture}[scale=.7]
\draw[densely dotted] (-3,-3) grid (3,3);
\foreach \p in {(2,0), (2,1), (2,2), (1,2)} {
 \fill \p circle (0.08);
}
\draw (2,0) -- (2,2) -- (1,2) -- (2,0);

\node[below right] at (2,0) {$\bar x_0$};
\node () at (0,-4) {(b)};
\end{tikzpicture}
\caption{(a) The digital topological group $G\subset (\Z^2,c_1)$ isomorphic to $\Z_{16}$, with points from the subgroup $N\cong \Z_4$ circled. (b) The quotient group $G/N$, with adjacencies given by $\bar c_1$\label{quotientfig}}
\end{figure}
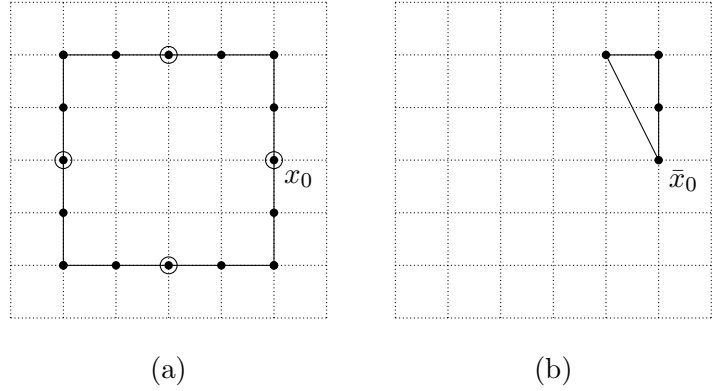

For sets $A,B\subset (X,\kappa)$ in some digital image, we will say that $A$ and $B$ are adjacent if there exists points $a\in A$ and $b\in B$ with $a\adjeq b$. In this case we write $A\adjeq B$. 

\begin{definition}
We call a $(\kappa, \lambda)$-continuous function $f : (X, \kappa) \to (Y, \lambda)$ a \emph{digital open map} when: if $z,w\in f(G)$ with $z\adj w$, then $f^{-1}(z)\adjeq f^{-1}(w)$.
\end{definition}

Thus an \emph{open map} is one for which $f^{-1}$ preserves adjacencies in the same way that a continuous map does. 
The definition of digital open map is related to continuity of the preimage $f^{-1}$ viewed as a multivalued map $f^{-1}:Y \to 2^Y;$ see \cite[Theorems 2.4 and 2.7]{bs16}. We use the term open map because it is required for Proposition \ref{LS102}, which in the classical theory requires that the map be open. 

The following is the digital version of the first isomorphism theorem from group (or module) theory.

\begin{thm}\label{LS102}
Let $f : (G, \kappa) \rightarrow (H, \lambda)$ be a digital topological group homomorphism with kernel $K$. If $f$ is a digital open map,
then there exists a unique digital topological group isomorphism 
$$
F : G/K \rightarrow {\rm Im} f
$$ 
such that 
$$
F \circ \pi (x) = f(x)
$$ 
for all $x \in G$, where 
$$
\pi : (G,\kappa) \rightarrow (G/N, \bar \kappa)
$$ 
is the $(\kappa, \bar\kappa)$-continuous function given by $\pi (x) = \bar x$.
\end{thm}

\begin{proof}
From the first isomorphism theorem from group theory, we see that there exists a unique isomorphism of groups
$$
\xymatrix@C=10mm @R=10mm{
F : G/K \ar@{->>}[r]^-{\cong} &{\rm Im} f \subseteq H
}
$$
such that the following diagram
$$
\xymatrix@C=20mm @R=15mm{
G \ar[d]_-{\pi} \ar@{->>}[r]^-{f} &{\rm Im} f \subseteq H  \\
G/K  \ar@{->>}[ur]_-{F} }
$$
is commutative in the category of groups and group homomorphisms in the sense that
$$
f(x) = F \circ \pi (x) = F (\bar x)
$$
for all $x \in G$, where $\Im f$ is a digital topological subgroup of $(H,\lambda)$. We need only show that $F$ is a digital isomorphism. 

To show $F$ is continuous, take $\bar x \adj_{\bar \kappa} \bar y$ in $(G/K, \bar \kappa)$. Then 
there are some $x\in \pi^{-1}(x)$ and $y\in \pi^{-1}(y)$ with $x\adj_\kappa y$. Since $f$ is continuous, we will have $f(x)\adj_\lambda f(y)$. Thus we have
$$
F (\bar x) = f(x) \adj_\lambda f(y) = F (\bar y);
$$
and so $F$ is continuous.

Now to show $F^{-1}$ is continuous, take $f(x)\adj_\lambda f(y)$ in $\Im f$, and we will show that 
$$
F^{-1}(f(x)) \adj_{\bar \kappa} F^{-1}(f(y))
$$
in $(G/K, \bar \kappa)$. Since $F^{-1}\circ f =\pi$, we must show that 
$$
\bar x \adj_{\bar \kappa} \bar y.
$$
Since $f$ is a digital open map and $f(x)\adj_\lambda f(y)$, we have $f^{-1}(f(x)) \adjeq f^{-1}(f(y))$, and thus there is some $a \in f^{-1}(f(x))$ and $b \in f^{-1}(f(y))$ so that
$$
\bar x = \pi (x) = \pi (a) = \bar a \adj_{\bar \kappa} \bar b = \pi(b) = \pi(y) = \bar y 
$$
in $(G/K, \bar \kappa)$ as required.
\end{proof} 


As the following example shows, not all continuous homomorphisms of digital topological groups are open maps, and the conclusion of Theorem \ref{LS102} may fail when the map is not open. 

\begin{exa}\label{nonopen}
Let 
$$
X := \{(1,0),(0,1),(-1,0),(0,-1)\}
$$
be a digital image, pictured in Figure \ref{nonopenfig}. 
\begin{figure}[h]
\begin{tikzpicture}
\draw[densely dotted] (-2,-2) grid (2,2);
\node[below right] at (1,0) {$(1,0)$};\
\node[above right] at (0,1) {$(0,1)$};
\node[below left] at (-1,0) {$(-1,0)$};
\node[below right] at (0,-1) {$(0,-1)$};
\foreach \p in {(1,0),(0,1),(-1,0),(0,-1)} {
 \fill \p circle (.10cm);
}
\end{tikzpicture}
\caption{The digital image $X$ in $\mathbb{Z}^2$ from Example \ref{nonopen}.\label{nonopenfig}}
\end{figure}
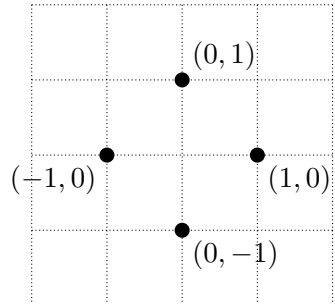
The digital image $(X,c_2)$ is a simple cycle of 4 points, while $(X,c_1)$ is a set of 4 never-adjacent points. Interpreting points of $X$ as complex numbers, the complex multiplication makes $X$ into a group isomorphic to $\Z_4$. It is easy to verify that this multiplication makes both $(X,c_1)$ and $(X,c_2)$ into $\NP_1$-digital topological groups. 

Let $f:(X,c_1) \to (X,c_2)$ be the identity map. Then clearly $f:X\to X$ is a continuous homomorphism, but it does not satisfy the open map condition. For example we have $(1,0) \adj_{c_2} (0,1)$, but $f^{-1}(1,0)$ and $f^{-1}(0,1)$ have no $c_1$-adjacent points. 

In this case the conclusion of Proposition \ref{LS102} will fail: the kernel of $f$ is trivial, so we have $(X,c_1)/\ker f \cong (X,c_1)$, while $\Im(f) \cong (X,c_2)$. But $f:(X,c_1) \to (X,c_2)$ is not a digital isomorphism.
\end{exa}

The following theorem shows that a group homomorphism (not assumed to be continuous) of digital topological groups is automatically continuous provided that it is continuous near the identity element. This is analogous to the classical fact that a homomorphism of topological groups will be continuous if it is assumed to be continuous in a neighborhood of the identity element. 
\begin{thm}
Let $G$ and $H$ be digital topological groups with identity elements $e_G$ and $e_H$, respectively, and let $f:G\to H$ be a group homomorphism with the property that: if $x\adj e_G$, then $f(x)\adjeq e_H$. Then $f$ is continuous and thus is a digital topological group homomorphism.
\end{thm}
\begin{proof}
Take $x,y\in G$, and we must show that $f(x)\adjeq f(y)$ in $H$. Since $x\adj y$, we can apply the isomorphism $\mu_{x^{-1}}$ to obtain $e_G \adj xy^{-1}$. Then by our assumption on $f$ we have $e_H \adjeq f(xy^{-1}) = f(x)f(y)^{-1}$. Then applying the isomorphism $\mu_{f(y)}$ gives $f(y)\adj f(x)$ as desired.
\end{proof}

We note that we can construct the category $\mathcal{DTG}$ whose object and morphism classes consist of digital topological groups and digital topological group homomorphisms, respectively.

\bigskip

\section{Classification of $\NP_2$-digital topological groups}\label{cdtg}

The $\NP_2$-continuity required for an $\NP_2$-digital topological group is very restrictive, and we can prove a complete classification of these groups. 
Recall that a \emph{cluster graph} \cite{SST} is a graph which is a disjoint union of complete graphs.

In this section we show that a digital image is an $\NP_2$-digital topological group if and only if it is a regular cluster graph. The result is based on the following lemma:

\begin{lem}\label{LS10}
If $G := (G,\kappa)$ is a connected $\NP_2$-digital topological group, then every element of $G$ is adjacent or equal to the identity element $e\in G$. 
\end{lem}
\begin{proof}
Let $x\in G$, and we will show that $x$ is adjacent or equal to $e$. To obtain a contradiction, assume that $x$ is not adjacent or equal to $e$, and let $(x_0,x_1,\dots,x_n)$ be a minimal path in $G$ from $x_0=e$ to $x_n=x$ with $n\ge 2$. Minimality of the path implies that $x_2$ is not equal or adjacent to $e$.

Since $(e,x_1)$ is adjacent to $(x_1,x_2)$ in $(G\times G,\NP_2(\kappa,\kappa))$, and the multiplication is $\NP_2$-continuous, we see that $x_1$ is adjacent or equal to $x_1x_2$. Applying the isomorphism $\mu_{x_1^{-1}}$ shows that $e$ is adjacent or equal to $x_2$, which is a contradiction.
\end{proof}

The result above can be improved:

\begin{thm}
If $G$ is a connected $\NP_2$-digital topological group, then $G$ is a complete graph.
\end{thm}
\begin{proof}
Let $x, y \in G$, and we must show that $x$ is adjacent or equal to $y$. Lemma \ref{LS10} above shows that $x$ and $y$ must be adjacent to $e$. Thus $(e,x)$ is $\NP_2$-adjacent to $(y,e)$, so multiplication shows that $x$ is adjacent to $y$ as desired.
\end{proof}

By Corollary \ref{componentsisom}, any digital topological group is a disjoint union of components, each isomorphic to $G_e$. Thus we obtain:

\begin{cor}
If $G$ is an $\NP_2$-digital topological group, then $G$ is a regular cluster graph. 
\end{cor}

We also have a converse to the above: any set which is a regular cluster graph can be given the structure of an $\NP_2$-digital topological group.

\begin{thm}
Let $X := (X,\kappa)$ be a regular cluster graph. Then there are operations $\mu_X:X\times X \to X$ and $\iota_X:X\to X$ which make $X$ into an $\NP_2$-digital topological group.
\end{thm}
\begin{proof}
Assume that $X$ consists of $k$ disjoint components, each a complete graph on $n$ vertices. We will construct coordinates for points of $X$ which label each point with a pair $(i,j)\in \Z_n\times \Z_k$, and show that, with respect to these coordinates, $X$ can be given the group structure of $\Z_n\times \Z_k$.

Choose some arbitrary point of $X$ and label it as $(0,0)$. Each of the other points in the component of $(0,0)$ will be labeled as $(i,0)$ for $i\in \Z_n$. Let $X_0,X_1,\dots, X_{k-1}$ be the components of $X$, labeled so that $(0,0)\in X_0$. For each $j\in \Z_k$, let $g_j:X_0\to X_j$ be some graph isomorphism, with $g_0$ being the identity map. Then for any $j \in \Z_k$, we label the points of $X_j$ as $$(i,j) = g_j(i,0).$$ 

Note that since $X$ is a cluster graph, two points $(i,j)$ and $(a,b)$ in $X$ are adjacent in $X$ if and only if they are in the same component, which is to say $$j=b.$$ 

Now we define the operations $\mu_X$ and $\iota_X$ so that $X$ becomes an $\NP_2$-digital topological group with group structure isomorphic to $\Z_n\times \Z_k$. We define:
\[ \mu_X((i,j),(a,b)) = (i+a,j+b) \text{ and } \iota_X(i,j) = (-i,-j). \]
Clearly these operations make $X$ into a group. By Lemma \ref{inversecontinuous} we need only show that $\mu_X$ is $\NP_2$-continuous. Take $((i,j),(a,b))$ to be $\NP_2$-adjacent to $((i',j'),(a',b'))$. This means that $j=j'$ and $b=b'$. Then their products will be $(i+a,j+b)$ and $(i'+a', j'+b')$ which are adjacent in $X$ because the second coordinates are equal. Thus $\mu_X$ is $\NP_2$-continuous as desired.
\end{proof}

\bigskip

\end{document}